\documentclass[11pt]{article}
\usepackage{geometry}                
\usepackage{amsmath,fourier}
\usepackage{amsthm}
\usepackage[parfill]{parskip}    
\usepackage{graphicx}
\usepackage{amssymb}
\usepackage{epigraph}
\usepackage{epstopdf}
\usepackage{subfig}
\usepackage{url}
\usepackage{array}
\usepackage{color}
\usepackage{tikz}
\usepackage{float}
\usepackage[]{algorithm2e}
\usepackage[]{algorithmic}

\makeatletter
\def\grd@save@target#1{%
  \def\grd@target{#1}}
\def\grd@save@start#1{%
  \def\grd@start{#1}}
\tikzset{
  grid with coordinates/.style={
    to path={%
      \pgfextra{%
        \edef\grd@@target{(\tikztotarget)}%
        \tikz@scan@one@point\grd@save@target\grd@@target\relax
        \edef\grd@@start{(\tikztostart)}%
        \tikz@scan@one@point\grd@save@start\grd@@start\relax
        \draw[minor help lines] (\tikztostart) grid (\tikztotarget);
        \draw[major help lines] (\tikztostart) grid (\tikztotarget);
        \grd@start
        \pgfmathsetmacro{\grd@xa}{\the\pgf@x/1cm}
        \pgfmathsetmacro{\grd@ya}{\the\pgf@y/1cm}
        \grd@target
        \pgfmathsetmacro{\grd@xb}{\the\pgf@x/1cm}
        \pgfmathsetmacro{\grd@yb}{\the\pgf@y/1cm}
        \pgfmathsetmacro{\grd@xc}{\grd@xa + \pgfkeysvalueof{/tikz/grid with coordinates/major step}}
        \pgfmathsetmacro{\grd@yc}{\grd@ya + \pgfkeysvalueof{/tikz/grid with coordinates/major step}}
        \foreach \x in {\grd@xa,\grd@xc,...,\grd@xb}
        \node[anchor=north] at (\x,\grd@ya) {\pgfmathprintnumber{\x}};
        \foreach \y in {\grd@ya,\grd@yc,...,\grd@yb}
        \node[anchor=east] at (\grd@xa,\y) {\pgfmathprintnumber{\y}};
      }
    }
  },
  minor help lines/.style={
    help lines,
    step=\pgfkeysvalueof{/tikz/grid with coordinates/minor step}
  },
  major help lines/.style={
    help lines,
    line width=\pgfkeysvalueof{/tikz/grid with coordinates/major line width},
    step=\pgfkeysvalueof{/tikz/grid with coordinates/major step}
  },
  grid with coordinates/.cd,
  minor step/.initial=.2,
  major step/.initial=1,
  major line width/.initial=1pt,
}
\makeatother

\theoremstyle{definition}

\theoremstyle{plain}
\newtheorem{theorem}{Theorem}
\newtheorem{proposition}{Proposition}
\newtheorem{lemma}{Lemma}

\theoremstyle{definition}

\theoremstyle{remark}

\newtheorem{remark}{Remark}

\newcommand{\D}{\mathbf{D}}
\newcommand{\tr}{\mathbf{trace}}

\newcommand{\Scal}{\mathcal{S}}

\newcommand{\Rbb}{\mathbb{R}}

\title{On the exact recovery of sparse signals via conic relaxations}
\author{Hongbo Dong\thanks{Department of Mathematics and Statistics, Washington State University, Pullman, WA 99163}}
\date{\today}                                           

\begin{document}
\maketitle

\begin{abstract}
In this note we compare two recently proposed semidefinite relaxations for the sparse linear
regression problem by Pilanci, Wainwright and El Ghaoui ("\textit{Sparse learning via boolean relaxations}", 2015) and Dong, Chen and 
Linderoth (''\textit{Relaxation vs. Regularization: A conic optimization perspective of statistical variable selection}", 2015).
We focus on the cardinality constrained formulation, and prove that the relaxation proposed by Dong, etc. is
theoretically no weaker than the one proposed by Pilanci, etc. Therefore any sufficient conditions of exact recovery derived by 
Pilanci can be readily applied to the Dong's relaxation, including their results on high probability recovery for Gaussian ensemble. 
Finally we provide empirical evidence that Dong's relaxation requires much fewer observations to guarantee the recovery of true support.
\end{abstract}

\section{Two convex relaxations for sparse linear regression}

Given a collection of observed sample points $(x_i, y_i) \in \Rbb^p \times \mathcal{Y}$, the goal 
of a sparse learning task is to learn a linear function $x \mapsto \beta^T x$ that is then used 
to predict an outcome of $y\in \mathcal{Y}$ for future/unseen data, where $\beta$ 
is restricted to have a small number of nonzero entries). 
Such a task can be modeled as the following cardinality constrained optimization problem
\begin{equation}\label{SpML}
\min_{\substack{\beta \in \Rbb^p, \\ \|\beta\|_0 \leq k} } 
\frac{1}{n} \sum_{i=1}^n f \left(\beta^T x_i; y_i\right).
\end{equation}

With the cardinality constraint, (\ref{SpML}) is usually highly nonconvex and difficulty
to solve to global optimality. The authors in \cite{PilanciWainwrightGhaoui2015} 
considered the following regularized version, 
\begin{equation}\label{SpMLreg}
\min_{\substack{\beta \in \Rbb^p, \\ \|\beta\|_0 \leq k} } 
\frac{1}{n} \sum_{i=1}^n f \left(\beta^T x_i; y_i\right)+ \rho\|\beta\|_2^2.
\end{equation}

One of the key results in \cite{PilanciWainwrightGhaoui2015} shows that (\ref{SpMLreg})
can be equivalently formulated as minimizing a convex function over a subset of binary vectors,
\begin{equation}\label{SpMLbin}
\min_{\substack{z \in \{0,1\}^p, \\ \sum_{j} z_j \leq k} } 
\ \ \ \underbrace{\max_{v \in \Rbb^n} \left\{-\frac{1}{2\rho} v^T X \D(z) X^T v - \sum_{i=1}^n f^*(v_i; y_i) \right\}}_{G(z)},
\end{equation}
where $G(z)$ is convex because as it is the max function of infinite many linear functions,
and $f^*(v,y) := \sup_{t\in \Rbb} \left\{st - f(t, y)\right\}$ is the conjugate function of $f(\cdot; y)$.

In this note we focus on the important special case of sparse linear regression, i.e., we consider
the following cardinality-constrained quadratic program,
\begin{equation}\label{l0_card}
\nu_{\ell0} := \min_{\beta} \ \  \frac{1}{2} \|X\beta - y\|_2^2 + \frac{1}{2}\rho\|\beta\|_2^2, \ \ s.t. \ \ \|\beta\|_0 \leq k,
\tag{$\ell_0:{card}$}
\end{equation}

The authors of \cite{PilanciWainwrightGhaoui2015} further proposed to relax the binary condition $z \in \{0,1\}^p$ in 
(\ref{SpMLbin}) to $z \in [0,1]^p$, and studied the conditions under which such a relaxation is exact.
When specialized to the sparse linear regression problem, the continuous relaxation takes
the following form of a semidefinite program,
\begin{equation}\label{SDP_PWG}
\begin{aligned}
\nu_{PWG} := \min_{t \in \Re, z \in [0,1]^p} \ \ & 0.5 t  \\
 s.t. \ \ & \begin{bmatrix}t & y \\ y & I_n + \frac{1}{\rho} X \D(z) X^T \end{bmatrix} \succeq 0,  \ \ \ e^T z \leq k,
\end{aligned}\tag{$SDP_{PWG}$}
\end{equation}
where $e$ is a vector with all entries 1 in proper dimension, and $\D(z)$ is a diagonal matrix whose entries are $z_i, i=1,...,p$.
It can also be equivalently written as the following compact form,
\[
\nu_{PWG} = \frac{1}{2} \min_{z \in [0,1]^p, e^T z \leq k} \ \ y^T \left(\frac{1}{\rho} X \D(z) X^T + I_n\right)^{-1} y.
\]

Following a different approach, authors of \cite{DongChenLinderoth2015} recently proposed another semidefinite relaxation
for sparse linear regression where the $\ell$-0 norm appears as a regularized term. When modified as a convex relaxation
for the cardinality constrained form (\ref{SpMLbin}), their proposed semidefinite relaxation is,
\begin{equation}\label{SDP_DCL}
\begin{aligned}
\nu_{DCL} := \min_{b \in \Rbb^p, B \in \Scal^p} \ \ & \frac{1}{2}\left\langle \begin{bmatrix}y^T y & -y^T X \\ -X^T y & \rho I_p +  X^T X \end{bmatrix}, 
\begin{bmatrix}1 & b^T \\ b & B\end{bmatrix}\right\rangle \\
s.t. \ \ & \begin{bmatrix}1 & b^T \\ b & B\end{bmatrix} \succeq 0 \\
&  \begin{bmatrix}z_i & b_i \\ b_i & B_{ii}\end{bmatrix} \succeq 0, \forall i,  \ \  \sum_{i=1}^p z_i \leq k.
\end{aligned}\tag{$SDP_{DCL}$}
\end{equation}

In this note we compare these two semidefinite relaxations. We show that the relaxation (\ref{SDP_DCL}) 
is no weaker than (\ref{SDP_PWG}) in this section. In section \ref{sec:cert} we establish
a result that characterizes a certificate of exactness for the convex relaxation (\ref{SDP_DCL}), hence extends
a key result in \cite{PilanciWainwrightGhaoui2015} to (\ref{SDP_DCL}). Section \ref{sec:empirical} concerns 
the probability of exact recovery for the case of Gaussian ensemble, where we show empirically (\ref{SDP_DCL})
can recover the true support of with much less data points.

We first state a technical lemma that will be used soon. 
\begin{lemma}\label{lem:ridge}
For any $X \in \Re^{n\times p}$ and $\rho > 0$, we have
\[
\min_{\beta \in \Rbb^p} \left\{\frac{1}{2} \|X\beta - y\|_2^2 + \frac{1}{2}\rho\|\beta\|_2^2 \right\}  \ = \
\frac{1}{2} y^T \left(\frac{1}{\rho} X X^T + I_n\right)^{-1} y
\]
\end{lemma}
\begin{proof}
Straightforward computation.
\end{proof}

By Lemma \ref{lem:ridge}, (\ref{SDP_PWG}) can be reformulated as
\begin{equation}\label{PWG:card:Poly}
\nu_{PWG} = 
\min_{z \in [0,1]^p, e^T z \leq k} \ \ 
\min_{\beta \in \Rbb^p} \ \  \frac{1}{2} \left\|X\sqrt{\D(z)}\beta - y\right\|_2^2 + \frac{1}{2}\rho
\left\|\beta\right\|_2^2,
\end{equation}
where $\sqrt{\D(z)}$ is a diagonal matrix with the i-th diagonal entry $\sqrt{z_i}$. 

\begin{proposition}\label{prop:dominate} $\nu_{\ell 0} \geq \nu_{DCL} \geq \nu_{PWG}$.
\end{proposition}
\begin{proof}
Suppose that $(\tilde{b}, \tilde{B}, \tilde{z})$ is optimal in (\ref{SDP_PWG}). Without loss of generality we may assume that $\tilde{z}_i  = \frac{\tilde{b}_i^2}{\tilde{B}_{ii}}$ for all $\tilde{B}_{ii} \neq 0$,
and $\tilde{z}_i = 0$ otherwise. Therefore $\tilde{z}_i \in [0,1], \forall i$.
Define $\tilde{\beta}$ as 
\[
\tilde{\beta}_i = \begin{cases}\tilde{z}^{-\frac{1}{2}}_i \tilde{b}_i, & \ if \ \tilde{z}_i > 0\\
0 & \ if \ \tilde{z}_i = 0.\end{cases}
\]
Then obviously
$\tilde{b} = \sqrt{\D(\tilde{z})} \tilde{\beta}$ and $(\tilde{\beta}, \tilde{z})$ is feasible in (\ref{PWG:card:Poly}). We have
\begin{align*}
\frac{1}{2}\left\langle \begin{bmatrix}y^T y & -y^T X \\ -X^T y & \rho I_n + 
 X X^T \end{bmatrix}, 
\begin{bmatrix}1 & \tilde{b}^{T} \\ \tilde{b} & \tilde{B}\end{bmatrix}\right\rangle
\geq &
\frac{1}{2}\left\langle \begin{bmatrix}y^T y & -y^T X \\ -X^T y &  
 X X^T \end{bmatrix}, 
\begin{bmatrix}1 & \tilde{b}^{T} \\ \tilde{b} & \tilde{b}\tilde{b}^T\end{bmatrix}\right\rangle + 
\frac{1}{2}\rho \tr(\tilde{B}) \\
 \geq & \frac{1}{2} \left\|X \sqrt{\D(\tilde{z})} \tilde{\beta} - y\right\|_2^2 + 
 \frac{1}{2} \rho \left\|\tilde{\beta}\right\|_2^2 \geq \nu_{PWG}.
\end{align*}
The first inequality is because of $\tilde{B} \succeq \tilde{b}\tilde{b}^T$. The second inequality 
is because of $\tilde{B}_{ii} \tilde{z}_i \geq \tilde{b}_i$, which implies
$\tilde{B}_{ii} \geq \tilde{\beta}_i^2$, and the final inequality is by the characterization (\ref{PWG:card:Poly}).
\end{proof}

As (\ref{SDP_DCL}) satisfies the Slater condition, strong duality holds and the dual of (\ref{SDP_DCL}) is

\begin{equation}\label{SDP_DCL:dual}
\begin{aligned}
\nu_{DCL} = \frac{1}{2}y^T y + \max_{\tau, \lambda, t, d}& \ \  - \frac{1}{2}\tau -  \frac{1}{2}k \lambda\\
s.t. & \ \ 
\begin{bmatrix}
\tau & -y^T X - t^T \\
-X^T y - t & X^T X + \rho I - \D(d)
\end{bmatrix} \succeq 0 \\
& \ \ \begin{bmatrix}
\lambda & t_i  \\
t_i & d_i
\end{bmatrix} \succeq 0 ,\forall i.
\end{aligned}\tag{$SDP_{DCL}: dual$}
\end{equation}

\section{Certificate of exactness}\label{sec:cert}
Proposition \ref{prop:dominate} implies that 
if $\nu_{PWG} = \nu_{\ell 0}$, then $\nu_{DCL} = \nu_{\ell 0}$.
Therefore all sufficient conditions for the exactness of (\ref{SDP_PWG}) readily carry over to (\ref{SDP_DCL}). 
Authors of \cite{PilanciWainwrightGhaoui2015} provided a characterization of a \textit{certificate of exactness} for the continuous relaxation
of (\ref{SpMLbin}), as well as a specialized result on (\ref{SDP_PWG}). We restate their characterization result in Theorem \ref{pwg-cert}, and provide 
a parallel result in Theorem \ref{thm:dualcert} for (\ref{SDP_DCL}).
\begin{theorem}[Corollary 2 in \cite{PilanciWainwrightGhaoui2015}]\label{pwg-cert}
The convex relaxation (\ref{SDP_PWG}) is exact if and only if there is a subset $S \subseteq \{1,...,p\}$, where $|S| \leq k$, such that there exists $\lambda \in \Rbb_+$,
\begin{align}
|X_j^T M y | > \lambda, \ \ & \forall j \in S, \mbox{ and } \label{pwg_1}\\
|X_j^T M y | \leq \lambda, \ \ & \forall j \notin S, \label{pwg_2}
\end{align}
where $X_j\in \Rbb^n$ is the j-th column of $X$, and $M := \left(I_n + \rho^{-1} X_S X_S^T\right)^{-1}$.
\end{theorem}
Using this result, the authors were able to prove a high-probability exact recovery condition for the special case of Gaussian ensembles. We leave the discussion
of Gaussian ensemble in the next section. Here we provide a parallel characterization of certificates of exactness for (\ref{SDP_DCL}).
\begin{theorem}\label{thm:dualcert}
Let $S \subseteq \{1,...,n\}$, $|S| = k$ and $z^*$ be a binary vector such that $z^*_i = 1, \forall i \in S$ and $z^*_i = 0, \forall i
\notin S$. Further let $b^*$ be the optimal solution of the ridge regression in the restricted subspace, i.e.,
\[
b^* \in \arg\min_{\beta\in \Rbb^p} \left\{ \|X\beta - y\|_2^2 + \rho\|\beta\|_2^2 \ \middle| \ \beta_j = 0, \forall j\notin S \right\}
\]
Then $(b^*, b^* b^{*T}, z^*)$ is optimal to (\ref{SDP_DCL})
 if and only if there exists a vector $\tilde{d} \in \Rbb^p_+$ and 
scalar $\tilde\lambda \in \Rbb_+$ such that
\begin{align}
\rho^{-1} X^T X +  I_p - \D(\tilde{d}) \succeq 0, &\label{cond:1}\\
\tilde\lambda = \tilde{d}_i \left(X_i^T M y \right)^2, &\qquad \forall i \in S, \label{cond:2}\\
\tilde\lambda \tilde{d}_i \geq  \left(X_i^T M y \right)^2, &\qquad \forall i\notin S \label{cond:3}
\end{align}
where $M := \left(I_n + \rho^{-1} X_S X_S^T\right)^{-1}$.
\end{theorem}
The proof of Theorem \ref{thm:dualcert} exploits the optimality conditions of (\ref{SDP_DCL}) and its dual, and is 
given in detail in the appendix section. We remark that one can directly show that the conditions in Theorem \ref{thm:dualcert}
are no stronger than those in Theorem \ref{pwg-cert}.
\begin{remark}
Suppose that $\lambda$ is the scalar such that
(\ref{pwg_1}) and (\ref{pwg_2}) hold, then (\ref{cond:1}) -- (\ref{cond:3}) hold for $\tilde{\lambda}$ and $\tilde{d}$, where 
\[
\tilde{\lambda} := \ \max \ \left\{ \left(X_j^T M y\right)^2 : i\in S\right\},  \tilde{d}_i =  \tilde{\lambda} \left(X_j^T M y\right)^{-2}, \forall i \in S,
\mbox{ and } \tilde{d}_i = 1,\forall i \notin S.
\] 
Note that $\tilde{d}_i \in [0,1]$ for all $i$ by construction. Therefore (\ref{cond:1}) holds. $(\ref{cond:2})$ and $(\ref{cond:3})$ are also valid by construction.
\end{remark}

\section{Empirical comparison on exact recovery rate for Gaussian ensemble}\label{sec:empirical}
In this section we consider the special case of Gaussian ensemble, where the design matrix $X \in \Rbb^{n\times p}$ is generated with i.i.d. N(0,1) entries.
A ``true" signal $\beta^*$ is generated to be $k$-sparse, i.e., it has only $k$ number of nonzero entries, and each nonzero entry is of the order $1/\sqrt{k}$. 
The response vector $y$ is generated by $y = X\beta^* + \epsilon$, where $\epsilon$ has i.i.d $N(0,\gamma^2)$ entries. 
The following result is established in \cite{PilanciWainwrightGhaoui2015}, which characterizes the size of $n$ needed to guarantee the exact recovery of the 
support of $\beta^*$ with high probability.

\begin{theorem}
There are constants $c_0$ and $c_1$, such that the following holds. Suppose that we are given a sample size 
$n > c_0\frac{\gamma^2 + \|\beta_S^*\|_2^2}{\beta_{min}^{*2}} \log p$, 
and that we solve (\ref{SDP_PWG}) with $\rho = \sqrt{n}$. Then with probability at least 
$1-2 e^{-c_1n}$, the relaxation (\ref{SDP_PWG}) is exact, i.e., $\nu_{PWG} = \nu_{\ell 0}$.
\end{theorem}
Here $\beta_{min}^*$ is the minimal nonzero entry (in absolute value) of $\beta^*$.  Note that Proposition \ref{prop:dominate} ensures that under the same 
conditions, $\nu_{DCL} = \nu_{\ell 0}$ with (at least the same) high probability. 

In the remaining part of this section we empirically evaluate the exact recovery for the Gaussian ensemble case. We compare the 
probabilities of exact recovery by (\ref{SDP_PWG}) and (\ref{SDP_DCL}) for various $n$ and $p$. To avoid potential numerical issues in solution precision, 
we exploit Theorem \ref{pwg-cert} and Theorem \ref{thm:dualcert} to directly search for the certificates. Such a strategy enables us to test whether 
(\ref{SDP_PWG}) and (\ref{SDP_DCL}) provide the global solution to (\ref{l0_card}) on large number of simulated data sets without explicitly solving the 
semidefinite relaxations many times.

Given simulated data $(X,y,\beta^*)$, let $S$ denote the support of $\beta^*$. Let $\rho > 0$ be fixed, it is straightforward to test whether conditions in 
Theorem \ref{pwg-cert} are satisfied. If so, then (\ref{SDP_PWG}) recovers the true support of $\beta^*$. The situation is slightly more complicated 
for (\ref{SDP_DCL}) and Theorem \ref{thm:dualcert}. Here we describe a bisection algorithm to search for the dual certificates $\tilde{\lambda}$ and 
$\tilde{d}$, provided that the support of $\beta^*$ is used as the index set $S$.

\subsection{A bisection algorithm to search for the dual certificates for (\ref{SDP_DCL})}

Without loss of generality we assume that $S = \{1,...,|S|\}$.
Firstly if $X_i^T M y = 0$ for some $i \in S$, then the convex relaxation (\ref{SDP_DCL}) is not exact unless the trivial case where $X^T M y = 0$ for all $i$.
Then without loss of generality we can assume that
$\tilde{d}_i = \tilde\lambda \left(X_i^T M y\right)^{-2}$, for all $i\in S$, and
$\tilde{d}_i = \tilde\lambda^{-1} \left(X_i^T M y\right)^{2}$, for all $i\notin S$.  Therefore the problem of testing (\ref{cond:1}) -- (\ref{cond:3}) 
is then equivalent to testing whether there exists $\tilde\lambda > 0$ such that the following function is nonpositive,
\begin{equation}\label{eq:f_lambda}
f(\tilde\lambda) := \lambda_{\max} \left\{\begin{bmatrix}D_S(\tilde\lambda) & 0 \\ 0 & D_{\bar{S}}(\tilde\lambda)\end{bmatrix} - \rho^{-1} X^T X - I_p\right\},
\end{equation}
where $\lambda\{\cdot\}$ is the largest eigenvalue function,  $D_S(\tilde\lambda)$ is a $|S|\times|S|$ diagonal matrix with diagonal entries 
$\tilde\lambda \left(X_i^T M y\right)^{-2}, i=1,...,|S|$,
and similarly $D_{\bar{S}}(\tilde\lambda)$ is a diagonal matrix with diagonal entries $\tilde\lambda^{-1} \left(X_i^T M y\right)^{2}, i =|S|+1,...,p$. 

Note that $f(\tilde\lambda)$ is a convex function when $\tilde\lambda>0$. This is because $\lambda_{\max}\{\cdot\}$ is a convex function
and non-decreasing in terms of diagonal entries, and $\tilde\lambda^{-1}$ is convex when $\tilde\lambda > 0$. It is known that a subgradient of 
$f(\tilde\lambda)$ can be computed from an eigenvector associated with the largest eigenvalue in (\ref{eq:f_lambda}). 
Indeed, let $u$ be such an eigenvector, then 
\[
h(\tilde\lambda) := \sum_{i\in S}\left(X_i^T M y\right)^{-2} u_i^2 - \lambda^{-2} \sum_{i \in \bar{S}} \left(X_i^T M y\right)^{2} u_i^2 \in \partial f(\tilde\lambda).
\]
In other words, given any $\hat{\lambda}>0$, $f(\lambda) \geq f(\hat\lambda) + h(\hat\lambda) (\lambda- \hat\lambda)$ for all $\lambda > 0$.
A final ingredient needed for a bisection algorithm is the initial interval. Consider the diagonal entries of the matrix in (\ref{eq:f_lambda}),
obviously if $\tilde\lambda \left(X_i M y\right)^{-2} - (\rho^{-1}X_i^T X_i + 1) \geq 0$ for some $i \in S$, or 
$\lambda^{-1} \left(X_i M y\right)^{2} - (\rho^{-1}X_i^T X_i + 1) \geq 0$ for some $i \notin S$, then $f(\tilde\lambda) \geq 0$. 
Therefore we can restrict ourself in a region such that $\tilde\lambda \left(X_i M y\right)^{-2} - (\rho^{-1}X_i^T X_i + 1) \leq 0$ for all $i \in S$,
and $\lambda^{-1} \left(X_i M y\right)^{2} - (\rho^{-1}X_i^T X_i + 1) \leq 0$ for all $i \notin S$.
This provides initial upper and lower bounds such that if there exists $\hat\lambda$ such that $f(\hat{\lambda}) < 0$, $\hat\lambda$ must be in
\[
\left[\max_{i \in \bar{S}} \left\{\left(X_i M y\right)^{2} (\rho^{-1}X_i^T X_i + 1)^{-1}\right\}, \min_{i\in S} \left\{\left(X_i M y\right)^{2} (\rho^{-1}X_i^T X_i + 1)\right\}\right].
\]
Then the problem of testing whether there is $\tilde{\lambda}$ such that $f(\tilde{\lambda})$ can be solved by the following bisection algorithm, 
\begin{enumerate}
\item Start with $\ell = \max_{i \in \bar{S}} \left\{\left(X_i M y\right)^{2} (\rho^{-1}X_i^T X_i + 1)^{-1}\right\}$ and 
$u =  \min_{i\in S} \left\{\left(X_i M y\right)^{2} (\rho^{-1}X_i^T X_i + 1)\right\}$;
\item Let $\hat{\lambda} = \frac{\ell+u}{2}$ and evaluate $f(\hat\lambda)$;
\item If $f(\hat\lambda) \leq 0$, return YES; otherwise compute $h(\hat\lambda)$;  
\item If $h(\hat\lambda) = 0$, return NO. If $h(\hat\lambda) > 0$, $u \leftarrow \hat{\lambda} -\frac{\hat\lambda}{h(\hat\lambda)}$; otherwise if $h(\hat\lambda) < 0$, 
$\ell \leftarrow \hat{\lambda} -\frac{\hat\lambda}{h(\hat\lambda)}$. If $u-\ell > \epsilon$, where $\epsilon$ is a fixed precision tolerance, then go to step 2. Otherwise return NO.
\end{enumerate}

\subsection{Numerical simulations}
Using this bisection algorithm, we conduct similar experiments as shown in Figure 1 of \cite{PilanciWainwrightGhaoui2015}. For each value of $p$ (denoted as $d$ in all plots), the true
sparsity is set as $\left\lceil \sqrt{p} \right\rceil$, and the number of data points $n=\alpha k \log(p-k)$. The true signal $\beta^{*}_i, (i=1,...k)$ is generated to be 1 or -1 with same probability. The 
Figures \ref{fig:1} through \ref{fig:6} show the exact support recovery rate for $\alpha \in [1,10]$, when $\rho$ are chosen to be $2 \sqrt{n}, 3\sqrt{n}, 4\sqrt{n}, 6\sqrt{n}, 8\sqrt{n}, 12\sqrt{n}$. 

The numerical simulation illustrates that (\ref{SDP_DCL}) can recover the support of true signals with significant less data points than that of (\ref{SDP_PWG}). Also the exact 
recovery rate of (\ref{SDP_DCL}) appears to be much less sensitive to the choice of $\rho$. This result further motivates us to study scalable approximate methods, such
as those based on low rank factorization of the matrix $B$, to solve (\ref{SDP_DCL}).

\begin{figure}[htbp]
\begin{center}
\includegraphics[scale=0.5]{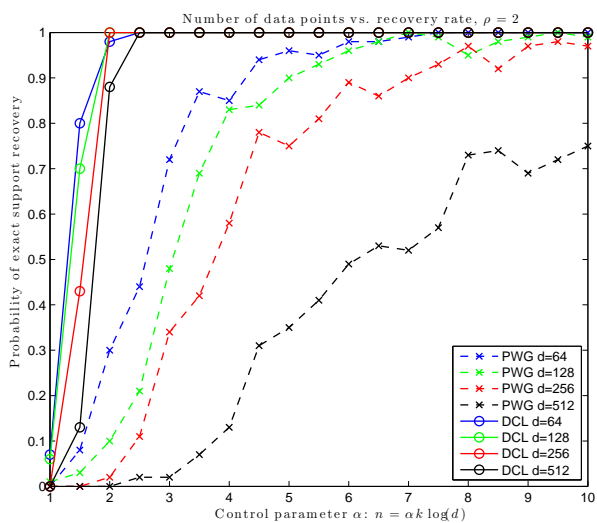}
\caption{Exact support recovery rate when $\rho = 2\sqrt{n}$.}
\label{fig:1}
\end{center}
\end{figure}

\begin{figure}[htbp]
\begin{center}
\includegraphics[scale=0.5]{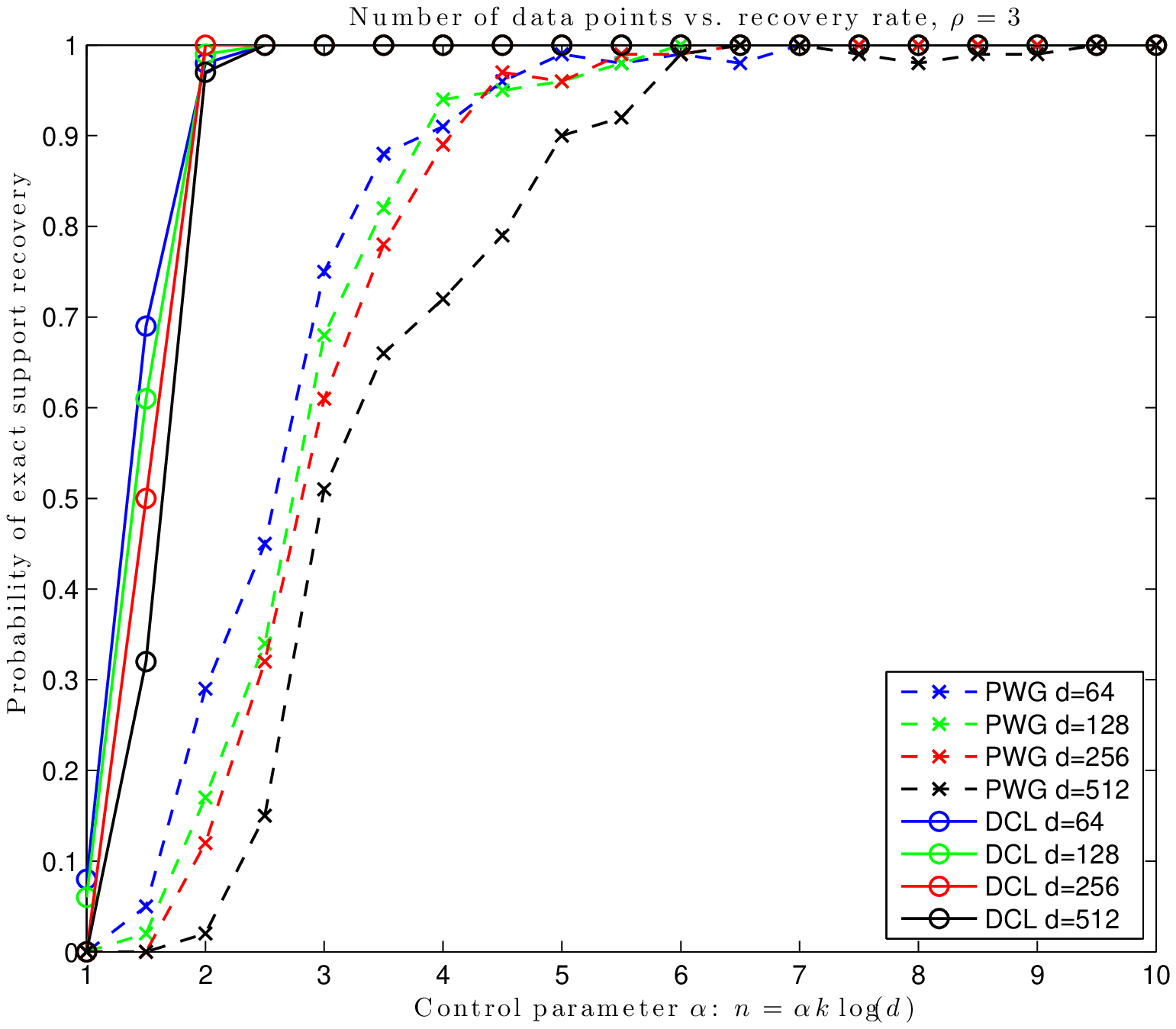}
\caption{Exact support recovery rate when $\rho = 3\sqrt{n}$}
\label{fig:2}
\end{center}
\end{figure}

\begin{figure}[htbp]
\begin{center}
\includegraphics[scale=0.5]{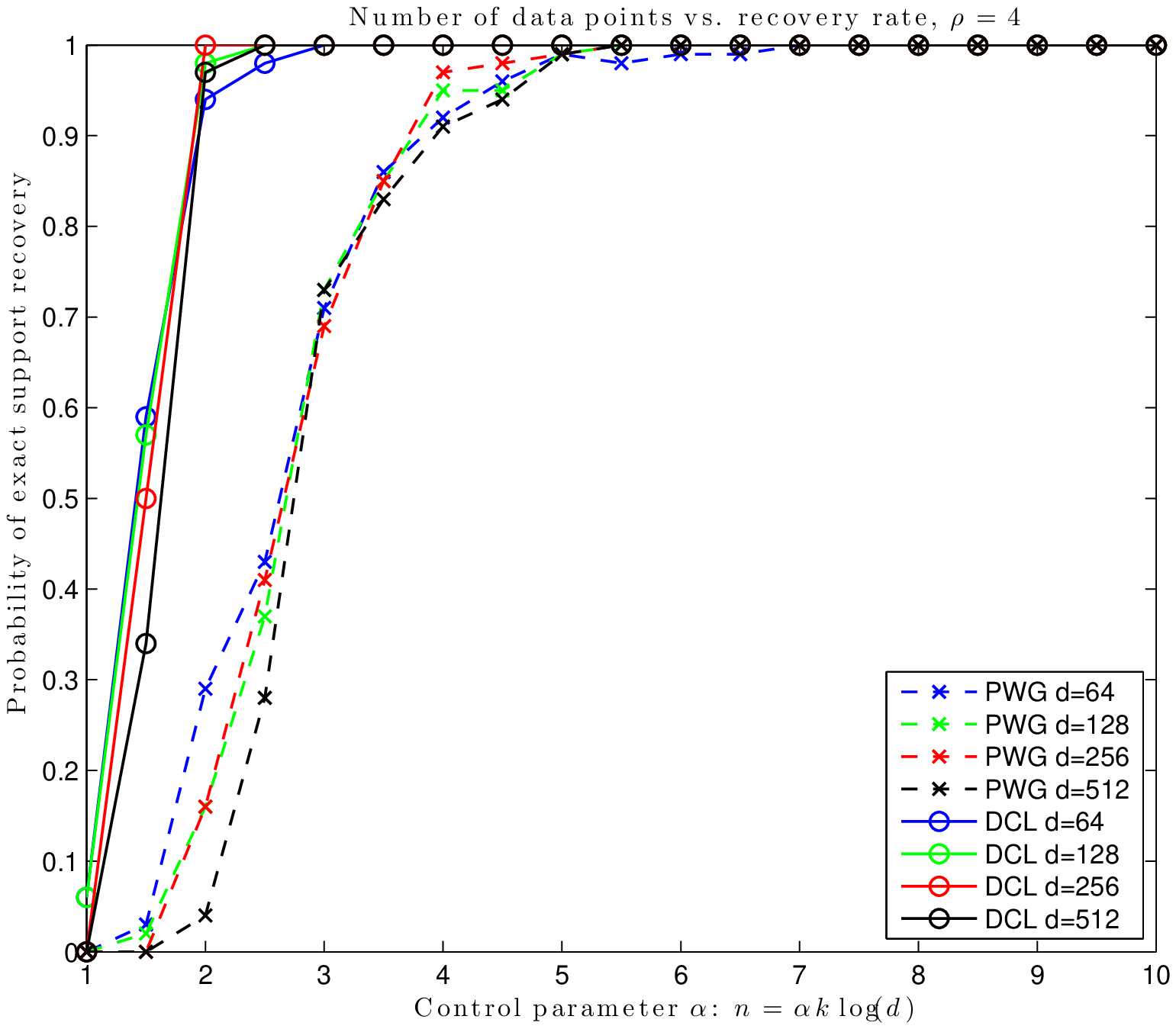}
\caption{Exact support recovery rate when $\rho = 4\sqrt{n}$}
\label{fig:3}
\end{center}
\end{figure}

\begin{figure}[htbp]
\begin{center}
\includegraphics[scale=0.5]{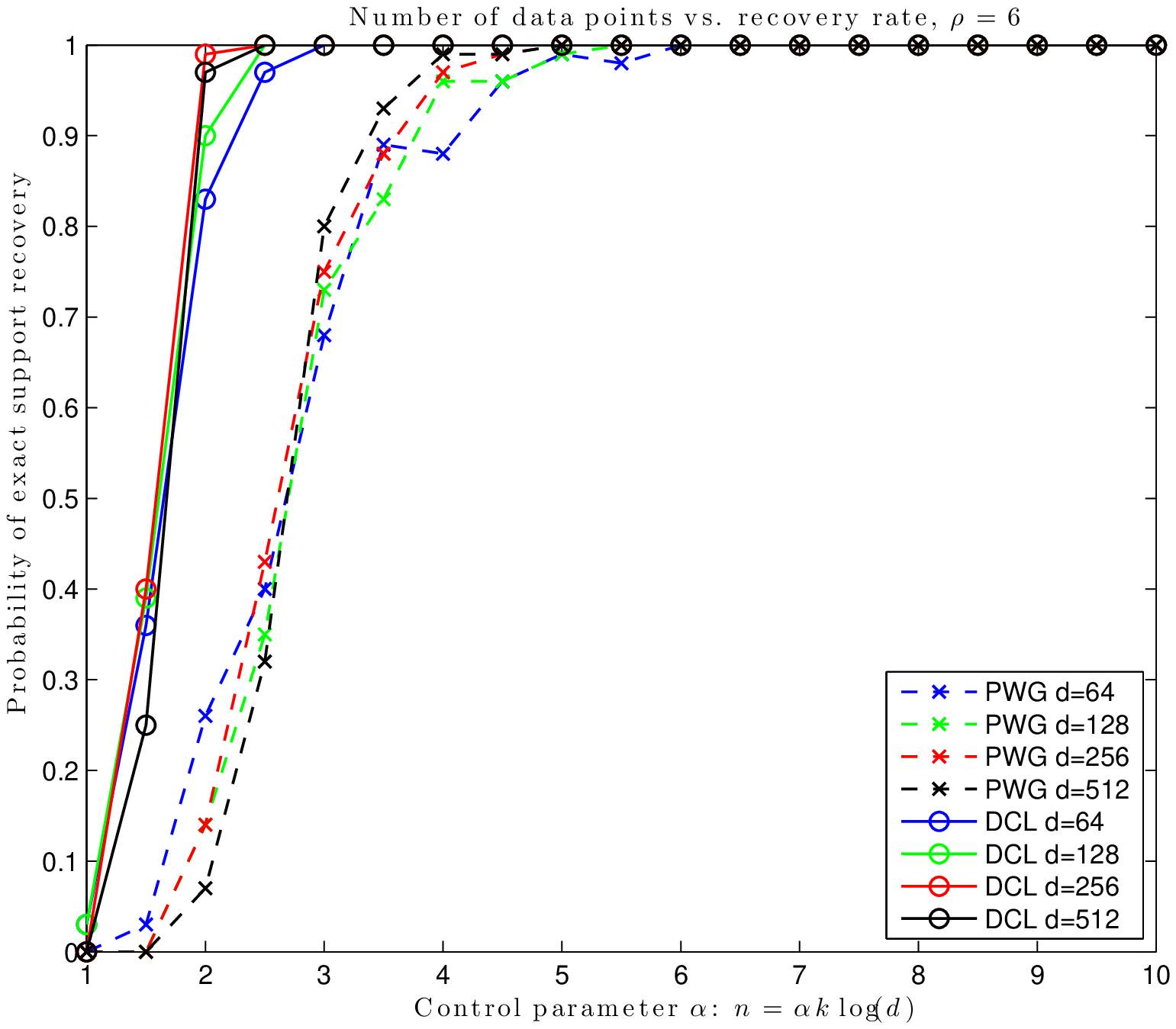}
\caption{Exact support recovery rate when $\rho = 6\sqrt{n}$}
\label{fig:4}
\end{center}
\end{figure}

\begin{figure}[htbp]
\begin{center}
\includegraphics[scale=0.5]{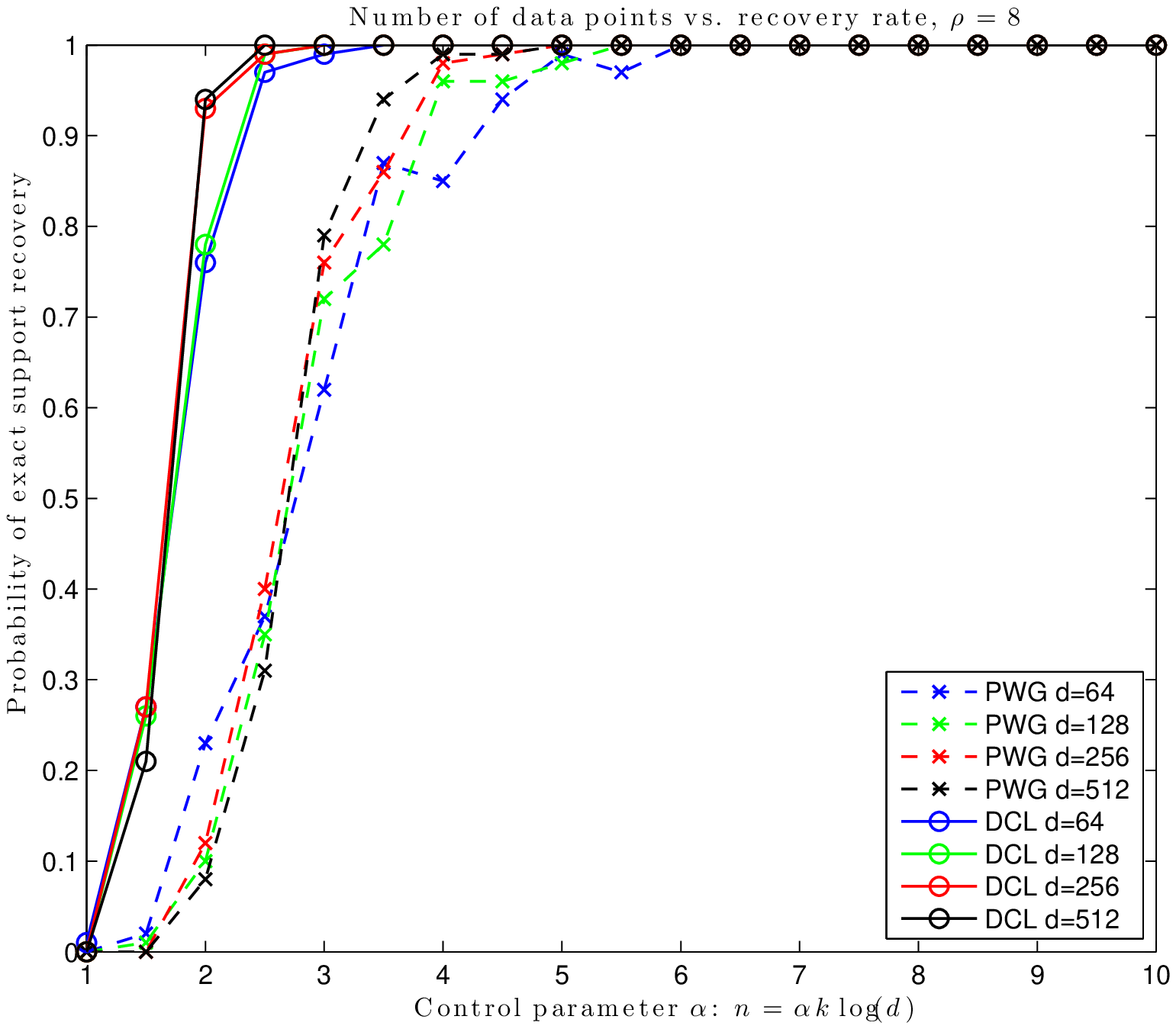}
\caption{Exact support recovery rate when $\rho = 8\sqrt{n}$}
\label{fig:5}
\end{center}
\end{figure}

\begin{figure}[htbp]
\begin{center}
\includegraphics[scale=0.5]{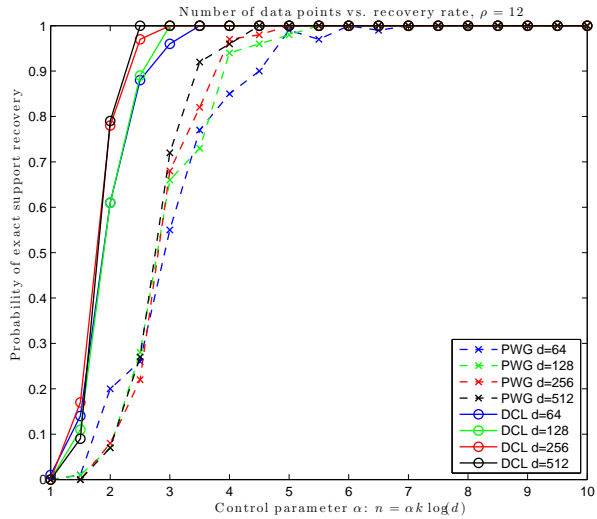}
\caption{Exact support recovery rate when $\rho = 12\sqrt{n}$}
\label{fig:6}
\end{center}
\end{figure}

\newpage

\bibliographystyle{plain}
\bibliography{SDPcomp}

\begin{thebibliography}{1}

\bibitem{DongChenLinderoth2015}
Hongbo Dong, Kun Chen, and Jeff Linderoth.
\newblock {Regularization vs. Relaxation: A conic optimization perspective of
  statistical variable selection}.
\newblock Submitted to Math. Prog. A, 2015.

\bibitem{PilanciWainwrightGhaoui2015}
Mert Pilanci, Martin~J. Wainwright, and Laurent~El Ghaoui.
\newblock {Sparse learning via Boolean relaxations}.
\newblock {\em Mathematical Programming (Series B)}, 151:63--87, 2015.

\end{thebibliography}

\section*{Appendix}
\textbf{Proof of Theorem \ref{thm:dualcert}.}
\begin{proof}
Without loss of generality we assume $S = \{1,...,k\}$.
When $z = z^*$, the constraints in (\ref{SDP_DCL}) enforce that $B_{ii} = b_i = 0, \forall i \notin S$. Further since
$z^*$ is binary, we may assume $B^* = b^* b^{*T}$ without loss of generality, where $b^*$ solves the restricted regression
problem:
\[
b_S^* \in \arg\min_{\beta\in \Rbb^{|S|}} \left\|X_S\beta - y\right\|_2^2 + \rho\|\beta\|^2 \ \ 
\Rightarrow \ \ b^* = \begin{bmatrix}\left(\rho I + X_S^T X_S \right)^{-1} X_S^T y \\ \mathbf{0}_{|S^c|}\end{bmatrix}
\]
and $\mathbf{0}_{|S|}$ is a $|S|\times 1$ zero vector.
By strong duality and the KKT conditions, there exists $b^*$, $B^*$ such that $(b^*,B^*,z^*)$ is optimal in
(\ref{SDP_DCL}) if and only if there exists dual variables $(\tau, \lambda, t, d)$ such that the first order
optimality condition holds
\begin{align}
\begin{bmatrix}
\tau & -y^T X - t^T \\
-X^T y - t & X^T X + \rho I - \D(d)
\end{bmatrix} \succeq 0 \label{opt:1}\\
 \begin{bmatrix}
\lambda & t_i  \\t_i & d_i \end{bmatrix} \succeq 0 ,\forall i  \label{opt:2} \\
\begin{bmatrix}
\tau & -y^T X - t^T \\
-X^T y - t & X^T X + \rho I - \D(d)
\end{bmatrix} \bullet 
\begin{bmatrix} 1& b^{*T} \\ b^*& b^*b^{*T}\end{bmatrix} = 0  \label{opt:3} \\
\begin{bmatrix}\lambda & t_i  \\t_i & d_i\end{bmatrix} 
\bullet
\begin{bmatrix} 1 & b^*_i \\ b^*_i& (b^*_i)^2\end{bmatrix} = 0, \ \ \forall i\in S \label{opt:4} 
\end{align}

In (\ref{opt:4}), only complementarity conditions for $i \in S$ are needed because for all $j \notin S$, $z_j^*$ equals 0, 
which implies that $b^*_j = 0$ by feasibility and thus the complementarity condition holds.
Now we aim to derive simpler conditions on the existence of such dual variables $(\tau,\lambda, t,d)$. 
With condition (\ref{opt:1}), (\ref{opt:3}) can be equivalently written as
\[
\begin{bmatrix}
\tau & -y^T X - t^T \\
-X^T y - t & X^T X + \rho I - \D(d)
\end{bmatrix}  
\begin{bmatrix} 1\\ b^*\end{bmatrix} = 0,
\]
which is further equivalent to the following two equations (\ref{tasd}) and (\ref{tauasd}), 
\begin{align}
 t &= (X^T X + \rho I -\D(d))b^* - X^T y \label{tasd}\\
\tau &= b^{*T} t + y^T X b^* = b^{*T} (X^T X + \rho I -\D(d)) b^*. \label{tauasd}
\end{align}
Now we exploit (\ref{tasd}) and (\ref{tauasd}) to eliminate $t$ and $\tau$ in (\ref{opt:1}), 
\[
\begin{bmatrix}
\tau & -y^T X - t^T \\
-X^T y - t & X^T X + \rho I_p - \D(d)
\end{bmatrix} 
= \begin{bmatrix}b^{*T} \\ I_p \end{bmatrix}\begin{bmatrix}X^T X + \rho I - \D(d)\end{bmatrix}
 \begin{bmatrix}b^* & I_p \end{bmatrix}.
\]
Therefore conditions (\ref{opt:1}) and (\ref{opt:3}) are equivalent to (\ref{tasd}), (\ref{tauasd}) and 
\begin{align}
X^T X + \rho I_p - \D(d) \succeq 0. \label{opt:smallpsd}
\end{align}
Now we consider conditions (\ref{opt:2}) and (\ref{opt:4}). Again with (\ref{opt:2}), (\ref{opt:4}) is equivalent to
\[
\begin{bmatrix}\lambda & t_i  \\t_i & d_i\end{bmatrix} 
\begin{bmatrix} 1 \\ b^*_i\end{bmatrix} = 0, \forall i\in S \ \Longleftrightarrow \ t_i = - d_i b^*_i, \lambda = -t_i b^*_i = d_i (b_i^*)^2,  \forall i\in S
\]
We claim that for all $i\in S$, $t_i = -d_i b^*_i$ is implied by (\ref{tasd}). Indeed, as $b^*$ minimizes the convex quadratic form in the restricted subspace
corresponding to $S$,
\begin{align*}
0 = \left. \frac{d}{d\beta_i} \right|_{\beta = b^*} \|X_S \beta - y\|_2^2 + \rho \|\beta\|^2 & = 2X_i^T (X_S b^* - y) + 2\rho b^*_i,
\ \ \ \forall i\in S, \\
&= 2 \left(X_i^T X b^* - X_i^Ty + \rho b_i^*\right).
\end{align*}
So the i-th row of (\ref{tasd}) can be equivalently written as,
\[
t_i = X_i^T (X b^* - y) + \rho b_i^* - d_i b_i^* = -d_i b_i^*, \ \ \forall i\in S.
\]
Therefore conditions (\ref{opt:2}) and (\ref{opt:4}) can be simplified as,
\begin{align*}
d_i  \geq 0,  \qquad & \forall i=1,...,p, \\
\lambda  = d_i (b_i^*)^2,  \qquad & \forall i \in S, \\ 
\lambda d_i  \geq t_i^2,   \qquad & \forall i \notin S.
\end{align*}
Note that for all $i \notin S$, $b^*_i = 0$. So by (\ref{tasd}), $t_i = X_i^T (X b^* - y)$ for all $i \notin S$.
Therefore the optimality conditions (\ref{opt:1}) -- (\ref{opt:4}) are equivalent to 
\[
(\ref{tasd}), \ (\ref{tauasd}), \ (\ref{opt:smallpsd}), \ \lambda = d_i (b_i^*)^2, \forall i\in S, \ \lambda d_i \geq \left(X_i^T (X b^* - y)\right)^2, \forall i\notin S.
\]
Note that (\ref{tasd}) and (\ref{tauasd}) simply state that and $t$ and $\tau$ are uniquely determined once $d$ is fixed, where $t$ and $\tau$
do not appear in other conditions. To complete the proof it suffices to prove two sets of equalities: 
\begin{equation}\label{finaleq1}
b^*_i = \rho^{-1} X_i^T (Xb^* - y), \qquad \forall i \in S,
\end{equation}
and 
\begin{equation}\label{finaleq2}
-X_i^T(Xb^* - y) = X_i^T ( \rho I + X_S X_S)^{-1} y, \qquad \forall  i.
\end{equation} 
Our conclusion then follows after a rescaling $\tilde{d} = d_i \rho^{-1}$ and $\tilde{\lambda} = \lambda \rho^3$.
Indeed, the equalities (\ref{finaleq1}) and (\ref{finaleq2}) can be proved by using the Sherman-Morrison-Woodbury formula,
\begin{align*}
X_i^T (Xb^* - y) &= X_i^T \left(X_S \left(\rho I + X_S^T X_S\right)^{-1} X_S^T y - y \right)  \\
&=  - X_i^T \left(I - X_S \left(\rho I + X_S^T X_S\right)^{-1} X_S^T \right) y \\
& = - X_i^T \left( I + \rho^{-1} X_S X_S^T \right)^{-1} y;  \qquad \forall i\\
\\
b_S^* &= \left(\rho I + X_S^T X_S\right)^{-1} X_S^T y \\
&=\left[\rho^{-1} I - \rho^{-2}X_S^T \left(I + \rho^{-1}X_S X_S^T\right)^{-1} X_S \right]X_S^T y \\
&=\rho^{-1}X_S^T \left(I + \rho^{-1}X_S X_S^T\right)^{-1} y.
\end{align*}
\end{proof}

\end{document}